\documentclass{article} 
\usepackage{style_files/iclr2021_conference,times}

\usepackage{amsmath,amsfonts,bm,amsthm}

\newtheorem{theorem}{Theorem}[section]
\newtheorem{lemma}[theorem]{Lemma}
\newtheorem*{remark}{Remark}

\def\eqref#1{equation~\ref{#1}}

\def\1{\bm{1}}

\def\rvepsilon{{\mathbf{\epsilon}}}

\def\rvalpha{{\boldsymbol{\alpha}}}

\def\rvepsilon{{\boldsymbol{\epsilon}}}

\def\rvI{{\mathbf{I}}}

\def\rvs{{\mathbf{s}}}

\def\rvx{{\mathbf{x}}}

\def\vtheta{{\bm{\theta}}}

\def\veta{{\bm{\eta}}}
\def\vepsilon{{\bm{\epsilon}}}
\def\vf{{\bm{f}}}

\def\vh{{\bm{h}}}

\def\vomega{{\bm{\omega}}}

\def\vs{{\bm{s}}}

\def\vT{{\bm{T}}}

\def\vv{{\bm{v}}}

\def\mR{{\bm{R}}}

\def\mT{{\bm{T}}}
\def\mU{{\bm{U}}}
\def\mV{{\bm{V}}}
\def\mW{{\bm{W}}}

\def\mPhi{{\bm{\Phi}}}

\def\mSigma{{\bm{\Sigma}}}

\DeclareMathAlphabet{\mathsfit}{\encodingdefault}{\sfdefault}{m}{sl}
\SetMathAlphabet{\mathsfit}{bold}{\encodingdefault}{\sfdefault}{bx}{n}

\newcommand{\E}{\mathbb{E}}

\newcommand{\R}{\mathbb{R}}

\newcommand{\Cov}{\mathrm{Cov}}

\DeclareMathOperator*{\argmax}{arg\,max}
\DeclareMathOperator*{\argmin}{arg\,min}

\usepackage{hyperref}
\usepackage{algorithm}
\usepackage{bbm}
\usepackage{graphicx}
\usepackage[noend]{algpseudocode}
\usepackage{url}
\usepackage{comment}
\usepackage{xcolor}
\usepackage{soul}
\usepackage{physics}

\title{
disentangling images with lie group transformations and sparse coding}

\author{Ho Yin Chau$^{1,2}$\;\;\;\; Frank Qiu$^{2,3}$\;\;\;\; Yubei Chen$^{1,2,4}$\;\;\;\; Bruno Olshausen$^{1,2}$ \\
$^1$ Redwood Center for Theoretical Neuroscience\\
$^2$ Berkeley AI Research\\
$^3$ Department of Statistics\\
University of California, Berkeley \\
$^4$ Facebook AI Research\\
\texttt{\{hchau630,frankqiu,yubeic,baolshausen\}@berkeley.edu}\\
}

\iclrfinalcopy 
\begin{document}

\maketitle

\begin{abstract}
Discrete spatial patterns and their continuous transformations are two important regularities contained in natural signals. Lie groups and representation theory are mathematical tools that have been used in previous works to model continuous image transformations. On the other hand, sparse coding is an important tool for learning dictionaries of patterns in natural signals. In this paper, we combine these ideas in a Bayesian generative model that learns to disentangle spatial patterns and their continuous transformations in a completely unsupervised manner. Images are modeled as a sparse superposition of shape components followed by a transformation that is parameterized by $n$ continuous variables. The shape components and transformations are not predefined, but are instead adapted to learn the symmetries in the data, with the constraint that the transformations form a representation of an $n$-dimensional torus. Training the model on a dataset consisting of controlled geometric transformations of specific MNIST digits shows that it can recover these transformations along with the digits.  Training on the full MNIST dataset shows that it can learn both the basic digit shapes and the natural transformations such as shearing and stretching that are contained in this data.
\end{abstract}

\section{Introduction}
A major challenge 
for both models of perception and
unsupervised learning is to form disentangled representations of image data, in which variations along the latent dimensions explicitly reflect factors of variation in the visual world. 
An important dichotomy among these factors of variation is the distinction between discrete patterns vs. continuous transformations \citep{mumford2010pattern}, the former referring to factors such as local shape features or objects and the latter typically referring to geometric transformations such as translation, scaling and rotation. 
Importantly, these factors are not overtly measurable but rather \textit{entangled} in the pixel values of an image.
Learning to disentangle the shapes and transformations inherent in image data is an important task that many previous works have attempted to address using architectures such as bilinear models, manifold models and modified VAEs  \citep{Tenenbaum_2000, Grimes_2005, Olshausen_2007, DiCarlo_2007, Cadieu_2012, Bengio_2013, Cheung_2014, Dupont_2018}.

One class of previous approaches has used Lie groups to model image transformations to varying degrees of generality \citep{Rao_1999, Miao_2007, Culpepper_2009, Sohl-Dickstein_2010, Cohen_2014, Cohen_2015, Gklezakos_2017}. A Lie group can be thought of as a parametric family of continuous transformations, and is a natural tool for modelling image transformations. However, all of these previous approaches focus solely on learning transformations but not the discrete patterns in the dataset. 

On the other hand, sparse coding is a widely known unsupervised algorithm for learning a dictionary of discrete spatial patterns from which images are composed \citep{Olshausen_1997}. Neurons in its hidden layer have been shown to recapitulate receptive field properties of V1 neurons after training on natural images, suggesting that the early human visual system may be employing the computational strategies of sparse coding. However, sparse coding does not model image transformations explicitly, and as a result, information about both shape and transformations are entangled in its representations.

Here we propose a novel unsupervised algorithm, Lie Group Sparse Coding (LSC), that combines the advantages of sparse coding and Lie group learning. In particular, our work is much inspired by the work of \citet{Cohen_2014}, which introduces the mathematical framework of representation theory to disentanglement learning.  We build on this framework by including a latent representation of shape via sparse coding.  The proposed algorithm infers both the sparse representation of shape and its transformation from an image, and it learns the shape dictionary and transformation operators from the data through an iterative process akin to expectation-maximization. We demonstrate the capacity of our model to disentangle representations of shape and transformation from controlled datasets where the ground truth is known, and from the full MNIST dataset where both the underlying shape categories and factors of variation due to style are unknown and must be learned from the data. 


\section{Preliminaries}

Sparse coding seeks to learn a dictionary of templates $\{\Phi_i\}$, such that each image $\rvI$ can be described by a sparse linear combination of these templates $ \rvI=\sum_i\Phi_i\,\alpha_i$. However, as mentioned in the introduction, sparse coding does not model image transformations explicitly. Transforming an image changes its sparse code $\rvalpha = [\alpha_1, \cdots, \alpha_K]^T$ in a non-equivariant manner, meaning form and transformations are entangled in the sparse code representation. To address this problem, we add an operator $\mT(\rvs)$ that explicitly models the transformations, so that images are now represented as transformations of patterns generated from the sparse coding model, $\rvI=\mT(\rvs)\mPhi \rvalpha$, where $\mPhi = [\Phi_1, \cdots, \Phi_K]$. 

Inspired by the work of \citet{Cohen_2014}, we choose to model the transformations $\mT(\rvs)$ as actions of compact, connected, commutative Lie groups on images. Such groups are equivalent to $n$-dimensional tori \citep{Dwyer_1998}. By the Peter-Weyl theorem, these transformations can be decomposed as $\mT(\rvs) = \mW \mR(\rvs)\mW^T$, where 
\begin{equation}
\label{eq:R}
\mR(\rvs) = \begin{bmatrix}
\cos(\vomega_1 ^T \rvs) & -\sin(\vomega_1 ^T \rvs) & & &\\
\sin(\vomega_1 ^T \rvs) & \cos(\vomega_1 ^T \rvs) & & &\\
                             &                            & \ddots & & \\
                             &                            &            & \cos(\vomega_L ^T \rvs) & -\sin(\vomega_L ^T \rvs) \\
                             &                            &            & \sin(\vomega_L ^T \rvs) & \cos(\vomega_L ^T \rvs) \\
\end{bmatrix},
\end{equation}

$\mW$ is an orthogonal matrix, and $\vomega_l \in \mathbb{Z}^n$. In practice, as we shall show in section \ref{section:5}, this parameterization supports the learning of various common transformations such as translation, rotation, shearing, stretching etc. This parameterization is nice in the sense that the dependence on the parameter $\rvs$ takes a simple form, namely a block diagonal matrix consisting of 2x2 rotation blocks, which allows for efficient inference and learning.

At this point, the reader may skip ahead to the next section, as the remainder of this section will give a more in-depth explanation of the theory behind the parameterization $\mT(\rvs) = \mW \mR(\rvs)\mW^T$. A large number of transformations - including rotations, rigid motion, and translations - can be understood as Lie groups, or more informally a group of continuous symmetries of a space. Some examples include the groups of $n$-dimensional rotations $SO(n)$ and rigid motions $SE(n)$. When a Lie group deforms data, such as rotating an image, this constitutes the action of that Lie group $G$ on the vector space of data. If this action is linear, this amounts to a representation of $G$ - an instantiation of $G$ as a set of linear operators. More formally, a representation of $G$ on a vector space $V$ is a group homomorphism $\rho: G \to GL(V)$ which maps each element of $G$ to an invertible linear transformation on $V$.\\
\\
In this paper we consider the representation of compact, connected, commutative (CCC) Lie groups. We choose CCC Lie groups because they have very simple representations. The Peter-Weyl theorem states that any unitary representation of a compact Lie group can be written as a direct sum of its irreducible representations; informally, a representation decomposes into its atomic parts - the irreducibles - which cannot be further decomposed. Since all irreducible representations of an $n$-dimensional torus (hence of CCC Lie groups, as they are equivalent) are of the form $\rho(e^{i\rvs}) = e^{i\vomega \cdot \rvs}$ for any $\vomega \in \mathbb{Z}^n$ \citep{Kamnitzer_2011}, where $e^{i\rvs} = [e^{i s_1}, e^{i s_2}, \cdots, e^{i s_n}]^T$ is a point on $\mathbb{T}^n$, any unitary representation a CCC Lie group is diagonal up to a unitary change of basis:
\begin{equation}
\label{eq:representation}
\rho(e^{i\vs}) = \mV\begin{bmatrix}
e^{i\vomega_1^T \vs} & & & \\
& e^{i\vomega_2^T \vs} & & \\
& & \ddots & \\
& & & e^{i\vomega_{D/2}^T \vs} \end{bmatrix}\mV^H \equiv \mV e^{\mSigma(\rvs)}\mV^H
\end{equation}
Since the data lives in the real vector space $\mathbb{R}^D$, we restrict $\rho$ to be real by choosing the $i\vomega_l$ in the diagonal matrix to come in purely imaginary conjugate pairs. After some simplification, this leads to the form $\mW\mR(\vs)\mW^T$  (see Appendix \ref{section:Appendix-A} for details). In fact, \textit{any} orthogonal representation of $\mathbb{T}^n$ can be written as $\mW\mR(\vs)\mW^T$ (see Theorem \ref{theorem:1} in Appendix \ref{section:Appendix-A} for details). In practice, we reduce the dimensionality of the model and improve efficiency by reducing the number of columns of $\mW$ to $2L < D$.\\

A minor downside of restricting our attention to CCC Lie groups is that it imposes several theoretical constraints on the class of learnable transformations: Compactness enforces the transformations to be periodic; connectedness precludes discrete transformations like reflections; commutativity excludes non-commutative transformations such as 3D rotations. Fortunately while many transformations violate these constraints in theory, we demonstrate in section \ref{section:5} that they can still be approximately learned in practice.

\section{Algorithm}
\subsection{Probabilistic Model}
Let $\rvI \in \R^D$ be the input image, where $D$ is even, and let $L \leq D/2$. We model $\rvI$ as
\begin{equation}
\rvI = \mW\mR(\rvs)\mW^T\mPhi \rvalpha + \rvepsilon
\label{eq:generative-model}
\end{equation}
where $\mW \in \R^{D \times 2L}$ is a matrix with orthonormal columns (i.e. $\mW^T\mW = \mathbbm{1}$), $\mPhi \in \R^{D \times K}$ is the dictionary with each column having unit L2 norm, and $\mR(\rvs)$ is the matrix defined in Eq. \ref{eq:R}. The random variables in the model are $\rvs$, $\rvalpha$, and $\rvepsilon$, which are all independent of each other. The transformation parameter $\rvs \in \R^n$ is a random vector whose components $s_i$ are i.i.d. with $s_i \sim \mathrm{Unif}(0,2\pi)$. The sparse code $\rvalpha \in \R^K$ is also a random vector whose components $\alpha_k$ are i.i.d. with $\alpha_k \sim \mathrm{Exp}(\lambda)$, the exponential distribution. The random noise $\rvepsilon$ is i.i.d. Gaussian with variance $\sigma^2$ and zero mean. Given an image $\rvI$, our goal is to infer the transformation parameters $\rvs$ and sparse code $\rvalpha$ according to their posterior distribution.  Given a large ensemble of images, our goal is to learn the parameters $\vtheta = \{\mW, \mPhi\}$ by maximizing their log-likelihood.  The procedures for inference and learning are presented in section \ref{section:3.2} below.

The weights $\vomega_l \in \mathbb{Z}^n$, which appear in the block diagonal rotational matrix $\mR(\rvs)$, are chosen such that if $\vomega_l$ is chosen then $-\vomega_l$ is omitted. This is because, as shown in the derivation of $\mT(\rvs) = \mW\mR(\rvs)\mW$ in Appendix \ref{section:Appendix-A}, each 2x2 block in $\mR(\rvs)$ is obtained by combining $\vomega_l$ terms with opposite signs in Eq. \ref{eq:representation}. A multiplicity $m \geq 1$ is then assigned to the weights, meaning each $\vomega_l$ is repeated $m$ times. Finally, we select the first $L$ $\vomega_l$ sorted by ascending frequency $||\vomega_l||_2$.

This model combines aspects of sparse coding \citep{Olshausen_1997} and the work by \citet{Cohen_2014} on learning irreducible representations of commutative Lie groups. If the term $\mW\mR(\rvs)\mW^T$ is replaced by the identity matrix, the model is identical to the sparse coding model of \citet{Olshausen_1997}. If $\mPhi\rvalpha$ is replaced by a transformed image $\rvI^\prime$, and $n = 1$ (scalar $s$), the model is identical to the one-parameter transformation model by \citet{Cohen_2014}.

\subsection{Inference and Learning}
\label{section:3.2}
\begin{algorithm}
\caption{Lie Group Sparse Coding Algorithm}
\label{alg}
\begin{algorithmic}[1]
    \State{$\vtheta = \{\mW, \mPhi\}  \gets \{\mW_0, \mPhi_0 \} $}
    \Comment{Initialize model parameters}
\While{$\mW, \mPhi$ not converged}
    \State{Get normalized image batch $\rvI$}
    \State{$\rvalpha \gets \rvalpha_0$}
    \Comment{Initialize sparse coefficients $\hat{\rvalpha}$}
    \For{$i \in \{1,\cdots,T\}$} 
    \Comment{Compute $\hat{\rvalpha} = \argmax_{\rvalpha} P_\vtheta (\rvalpha|\rvI)$}
        \State{Compute $P_\vtheta(\rvs|\rvI,\rvalpha)$}
        \State{$\Delta \rvalpha \gets \E_{\rvs \sim P_\vtheta(\rvs|\rvI,\rvalpha)}[\nabla_{\rvalpha} \ln P_\vtheta(\rvI|\rvs,\rvalpha)] + \nabla_{\rvalpha} \ln P_\vtheta(\rvalpha)$}
        \Comment{Compute gradient for $\rvalpha$}
        \State{$\rvalpha \gets \text{FISTA update}(\rvalpha, \Delta \rvalpha)$}
        \Comment{Update $\rvalpha$ using FISTA}
    \EndFor
    \State{$\hat{\rvalpha} \gets \rvalpha$}
    \State{Compute $P_\vtheta(\rvs|\rvI,\hat{\rvalpha})$}
    \State{$\Delta \mPhi \gets \E_{\rvs \sim P_\vtheta(\rvs|\rvI,\hat{\rvalpha})}[\nabla_\mPhi \ln P_\vtheta(\rvI|\rvs,\hat{\rvalpha})]$}
    \Comment{Compute approximate gradient for $\mPhi$}
    \State{$\mPhi \gets \text{Normalize}(\mPhi + \Delta \mPhi)$}
    \Comment{Update $\mPhi$ and normalize columns}
    \State{$\Delta \mW \gets \E_{\rvs \sim P_\vtheta(\rvs|\rvI,\hat{\rvalpha})}[\nabla_\mW \ln P_\vtheta(\rvI|\rvs,\hat{\rvalpha})]$}
    \Comment{Compute approximate gradient for $\mW$}
    \State{$\mW \gets  \text{RiemannianAdam}(\mW,\Delta \mW)$} \Comment{Update $\mW$ with RiemannianAdam optimizer}
\EndWhile
\end{algorithmic}
\end{algorithm}

The inference and learning procedure is outlined in Algorithm \ref{alg}. The general idea is as follows: to learn the model parameters $\vtheta = \{\mW, \mPhi\}$, we perform gradient ascent on their log-likelihood using the approximate gradient 
\begin{equation}
\nabla_\vtheta \ln P_\vtheta(\rvI) \approx \E_{\rvs \sim P_\vtheta(\rvs|\rvI,\hat{\rvalpha})}[\nabla_\vtheta \ln P_\vtheta(\rvI|\rvs,\hat{\rvalpha})]
\label{eq:gradient_theta}
\end{equation}
where $\hat{\rvalpha} = \argmax_\rvalpha P_\vtheta(\rvalpha|\rvI)$ is the MAP estimate of the hidden variable $\rvalpha$ (see Appendix \ref{section:Appendix-C} for derivation details). Notice moreover that 
\begin{equation}
\nabla_\vtheta \ln P_\vtheta(\rvI | \hat{\rvalpha}) = \E_{\rvs \sim P_\vtheta(\rvs|\rvI,\hat{\rvalpha})}[\nabla_\vtheta \ln P_\vtheta(\rvI|\rvs,\hat{\rvalpha})]
\label{eq:gradient_theta_exact}
\end{equation}
and that the right hand side Eq. \ref{eq:gradient_theta_exact} is just our approximate gradient in Eq. \ref{eq:gradient_theta}.
Hence, another interpretation is that we replaced the original objective function $\ln P_\vtheta(\rvI)$ by the approximate objective function $\ln P_\vtheta(\rvI|\hat{\rvalpha})$, which is much more computationally tractable than the original (see Appendix \ref{section:Appendix-B} for an explicit formula for $\ln P_\vtheta(\rvI|\rvalpha)$, and Appendix \ref{section:Appendix-C} for derivation details). The approximation step assumes that the posterior distribution $P(\rvalpha | \rvI)$ is sharply peaked around $\hat{\rvalpha}$, meaning $P(\rvalpha | \rvI) \approx \delta(\rvalpha - \hat{\rvalpha})$, where $\delta(\rvx)$ is the Dirac delta function. This is the same approximation used in the sparse coding algorithm by \citet{Olshausen_1997}. Also note the similarity between this approach and the EM algorithm, as each gradient step partially maximizes the expectation of the log likelihood with respect to the posterior distribution of the hidden variable given the current parameters. On the other hand, the MAP estimate $\hat{\rvalpha}$ is computed by gradient ascent on the log-posterior, whose gradient is computed via
\begin{equation}
\nabla_{\rvalpha} \ln P_\vtheta(\rvalpha|\rvI) = \E_{\rvs \sim P_\vtheta(\rvs|\rvI,\rvalpha)}[\nabla_{\rvalpha} \ln P_\vtheta(\rvI|\rvs,\rvalpha)] + \nabla_{\rvalpha} \ln P_\vtheta(\rvalpha)
\label{eq:gradient_alpha}
\end{equation}
(see Appendix \ref{section:Appendix-C} for derivation details). Note that both Eq. \ref{eq:gradient_theta} and \ref{eq:gradient_alpha} require inferring the posterior distribution of the transformation variable $P_\vtheta(\rvs|\rvI,\rvalpha)$. We show how this posterior distribution can be computed in Appendix $\ref{section:Appendix-B}$.



One may wonder whether the posterior distribution $P(\rvs|\rvI,\rvalpha)$ can be approximated by $\delta(\rvs - \hat{\rvs})$, just as we did for Eq. \ref{eq:gradient_theta}. This would allow us to avoid computing the full distribution and use the point estimate $\hat{\rvs} = \argmax_{\rvs} P_\vtheta(\rvs|\rvI,\rvalpha)$ (optimized using gradient descent) in order to update the model parameters. We find empirically that using this approach leads to worse convergence of the model parameters, and we believe there are two reasons for this: first, during the initial stages of training, the posterior distribution of $\rvs$ has many local extrema, and hence using a single point estimate $\hat{\rvs}$ is a bad approximation; second, the presence of many local extrema during initial stages means it is easy to get stuck in a local minimum using gradient descent. Furthermore, when the number of transformation parameters is small (which is the case for the  experiments in this paper), it is faster to simply compute the full distribution of $\rvs$ than performing gradient descent to find $\hat{\rvs}$, as the full distribution can be computed in a highly parallelized manner but gradient descent cannot.

Computation of the gradients involves the expectation term $\bar{\mR} = \E_{\rvs \sim P_\vtheta(\rvs|\rvI,\rvalpha)}[\mR(\rvs)]$ (see Appendix \ref{section:Appendix-C} for details), which is obtained by numerically integrating $\int_\rvs P_\vtheta(\rvs|\rvI,\rvalpha) \mR(\rvs)$ with $N$ samples along each dimension. An efficient way of computing this quantity using Fast Fourier Transform is detailed in \citet{Cohen_2015}, although numerical integration is adequate for our purposes.

We used FISTA to greatly speed up the inference of $\rvalpha$ by around a factor of 10 \citep{Beck_2009}. While the objective $\ln P_\vtheta(\rvalpha|\rvI)$ is not guaranteed to be convex in $\rvalpha$, which violates one of the theoretical assumptions of FISTA, we find that in practice $\rvalpha$ always converges well. Details of the usage of FISTA in LSC is provided in Appendix \ref{section:Appendix-D}.

As mentioned in section 3.1, there are two hard constraints on the model parameters $\mW, \mPhi$: the columns of $\mW$ must be orthonormal, and the columns of $\mPhi$ must have unit norm. The constraint on $\mW$ comes from the fact that we only want to learn orthogonal representations, while the unit norm constraint on $\mPhi$ prevents $\mPhi$ from growing without bound (see \citet{Olshausen_1997}). We optimized $\mPhi$ by performing projected gradient descent, normalizing each column of $\mPhi$ after each gradient step. For $\mW$, we used the Riemannian ADAM optimizer to optimize $\mW$ on the Stiefel manifold (the manifold of matrices with orthogonal columns). We used the Riemmanian ADAM implementation from the python package geoopt \citep{geoopt2020kochurov}. Using Riemannian ADAM instead of a simple projected gradient descent, as is done in \citet{Cohen_2014}, was empirically found to speed up convergence by around 3-4 times.

\section{Related Works}
\citet{Gklezakos_2017} proposed the Transformational Sparse Coding algorithm (TSC) that, like LSC, combines ideas from Lie group theory with sparse coding. A significant difference, however, is that in TSC the group representation is fixed manually rather than learned; more concretely, the transformations in the generative model are fixed to be the family of 2D affine transformations, whereas LSC allows for the learning of any transformations as long as they form a representation of a CCC Lie group. This limits the ability of TSC to adapt to the transformations in image data. Another difference is that in TSC images are modelled as being a combination of transformed variants of root templates, whereas in LSC images are modelled as transformed versions of a combination of templates. In other words, TSC focuses on the transformations of local features, whereas LSC focuses on the global image transformations.

The inference and learning of transformations in our algorithm is based on the TSA algorithm (Toroidal Subgroup Analysis) by \citet{Cohen_2014} which learns the representation of a one-parameter subgroup of the maximal torus. There are two main differences: first, TSA learns image transformations given pairs of images with the same shape, whereas LSC learns both form and transformation together without being given any information about the shape of the images; second, LSC generalizes the one-parameter subgroup representation learned in TSA to the representation of an \textit{arbitrary} $N$-dimensional torus, allowing for the learning of a wider variety of transformations. \citet{Cohen_2015} later extended their work to learning 3D object rotations using the group representation of $SO(3)$, although it again only learns the transformations but not the discrete spatial patterns.

A large body of work has appeared in recent years that modifies existing deep neural network architectures such as GAN and VAE to learn disentangled representations in an unsupervised manner \citep{Cheung_2014, Chen_2016, Higgins_2017, Dupont_2018, Kim_2018, Chen_2018}. While these models are more general and potentially more powerful than LSC due to their many convolutional layers, they are also substantially more complex (in terms of number of layers) than LSC which requires only one layer for transformation and one layer for sparse coding. The compactness of LSC is due to the fact that we explicitly designed a layer to model Lie group transformations. We speculate that a model with multiple such transformation layers could capture a broader range of image transformations than a generic multilayer convnet.

There are also important works in supervised deep learning that include a specialized module to handle image transformation. Spatial Transformer Networks by \citet{Jaderberg_2015} uses a differentiable transformer module that models affine transformations. Capsules use a group of neurons to collectively encode the probability of an object's presence and the pose/transformation of such an object \citep{Hinton_2011, Sabour_2017, Hinton_2018}. Representation theory is used to develop new neural network layers that are equivariant to input transformations \citep{Cohen_2016, Cohen_2017, Cohen_2018, Cohen_2019}.

LSC grew out of works on separating form and transformation using bilinear models. Bilinear models assume a generative process in which two vectors, one encoding form and the other encoding transformation, combine bilinearly, meaning the output is linear with respect to both vectors \citep{Rao_1998, Tenenbaum_2000, Grimes_2005, Olshausen_2007}. Some bilinear models do not explicitly learn a transformation operator, but instead learn transformed versions of shape templates, and as a result the transformations learned will not easily generalize to new shapes \citep{Tenenbaum_2000, Grimes_2005}. Other bilinear models do learn an explicit transformation operator, but because of the difficulty of inferring and learning large transformations, only local transformations are learned \citep{Rao_1998, Olshausen_2007}.

Attempts to learn larger transformations led to works on learning Lie group transformations. Typically, these transformations are learned from pairs of transformed images or a sequence/set of transforming images \citep{Rao_1999, Miao_2007, Culpepper_2009, Sohl-Dickstein_2010}. Early works learn the generator, or the Lie algebra, directly, but because of the computational intractability of gradient descent with respect to the matrix exponential, first-order Taylor expansion of the matrix exponential is used during learning, which limits the ability to learn from images with large transformations \citep{Rao_1999, Miao_2007}. An important technique discovered by \citet{Sohl-Dickstein_2010} is to diagonalize the generator, which allows for tractable Lie group learning from large transformations. This idea was formalized and generalized by \citet{Cohen_2014} using representation theory, allowing for the learning of complex, large transformations such as 3D rotation in a later paper \citep{Cohen_2015}. This was an important generalization, since describing the set of transformations on images as a representation of a Lie group rather than a Lie group itself, as is done in previous work, allows for the use of representation theory to simplify computations. For instance, the representation theory of the $N$-dimensional torus is used in deriving the simple parameterization of the transformation operator in this paper.

\section{Experiments}
\label{section:5}
To demonstrate that our algorithm can successfully disentangle different form and transformation factors, we first train the model on two synthetic datasets in which the generative models are fully known. We set $K = 10$ and $n = 2$, meaning there are $10$ dictionary templates and $2$ latent dimensions for the transformation parameter $\rvs$. In the first dataset, we select one image from each of the 10 digit classes in 28x28 MNIST, then apply $6000$ random 2D translations to each of the $10$ selected images, totalling $60000$ images. Both vertical and horizontal translations are drawn uniformly between $-7$ and $7$ pixels.  In the second dataset, instead of 2D translations, we apply $6000$ random rotations and scaling to the $10$ images. Rotation is drawn uniformly between $-75^\circ$ and $75^\circ$, while scaling is drawn uniformly between $0.5$ and $1.0$. Figure \ref{fig:1} shows 80 images from each dataset.

\begin{figure}[h]
\begin{center}

\includegraphics[width=\textwidth]{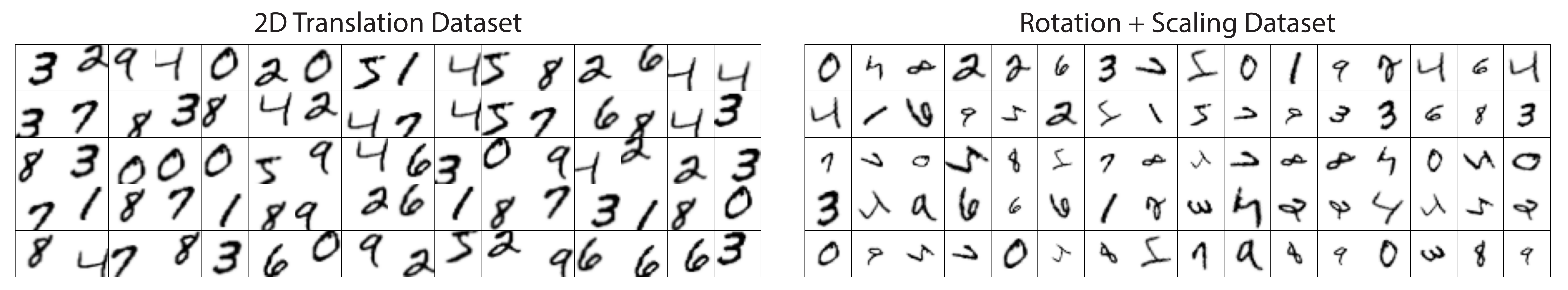}
\end{center}
\caption{80 example images from each of the two synthetic datasets}
\label{fig:1}
\end{figure}

For each dataset, LSC is able to learn the $10$ digits as well as the two operators that generated it (training details in Appendix \ref{section:Appendix-E}). Figure \ref{fig:2} shows the learned $W$ matrices, while the learned dictionary $\mPhi$ can be found in the middle of figure \ref{fig:3}. Notice that each of the learned dictionary template $\mPhi_i$ corresponds to one of the digits. Latent traversals of the two operators are shown at the bottom of figure \ref{fig:3}, in which we select 5 random images from the test set and apply the learned operator $\mT(\rvs)$ with varying $\rvs$ to those images. It is clear from the figure that the learned transformations are exactly the 2D translation operators and the rotation + scaling operators respectively. Strikingly, even though the rotation + scaling dataset contains only rotations between $-75^\circ$ and $75^\circ$, the model learns the full $360^\circ$ rotation. This ability to generalize and extrapolate correctly the transformation present in the dataset is a feature of the Lie group structure that is built into LSC. 

One might notice that there is a slight mixture of rotation and scaling in the latent traversal plots in figure \ref{fig:3}, which may seem to suggest that the algorithm failed to disentangle rotation and scaling completely. However, we note that there is no reason to require $\rvs_1$ and $\rvs_2$ to learn to parameterize pure rotation and pure scaling respectively, since learning to parameterize a linear combination of rotation and scaling also allows the network to perform arbitrary rotations and scalings perfectly. 

\begin{figure}[!htb]
\begin{center}
\includegraphics[width=\textwidth]{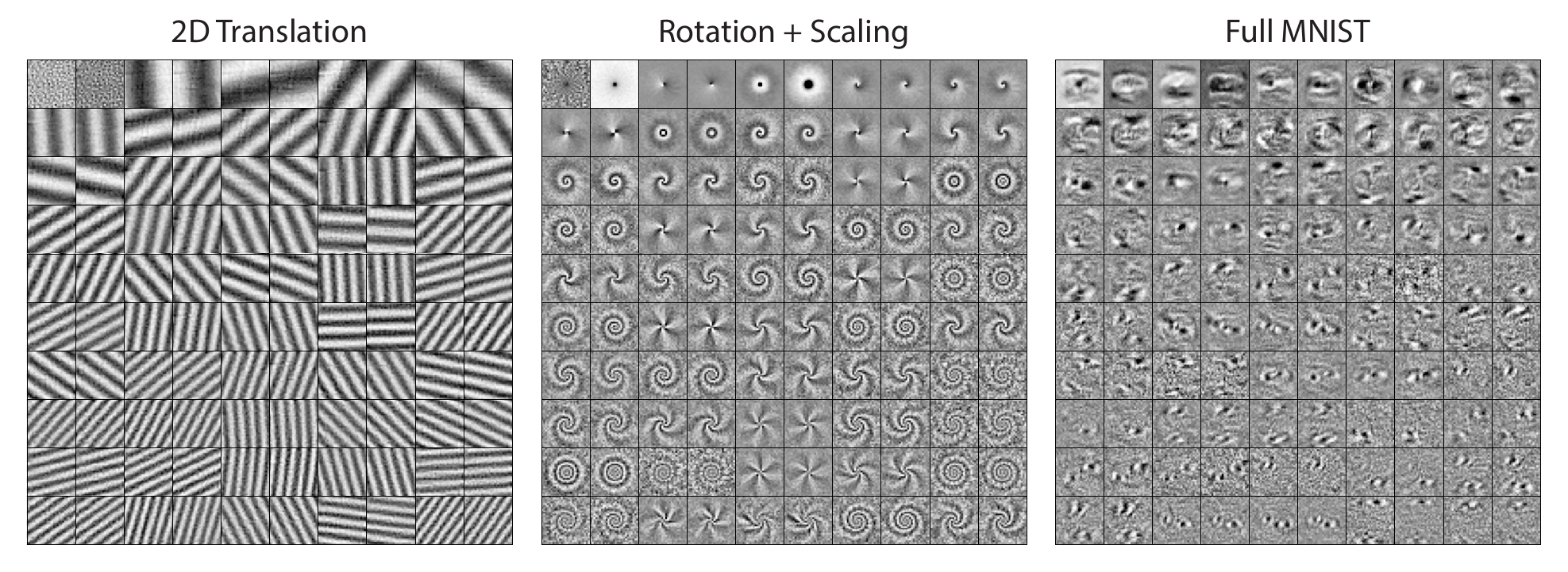}
\end{center}
\caption{The first $100$ columns of $\mW$ learned on the three different datasets. Each image shows a column of $\mW$, and are ordered by increasing values of $||\vomega||_2^2$}
\label{fig:2}
\end{figure}

The inference process is demonstrated at the top of figure \ref{fig:3}. An image $\rvI$ is given to the network, which then performs the inference procedure given in section 3.2 to yield the MAP estimate of the sparse coefficients $\hat{\rvalpha}$ and the posterior distribution of the transformation parameter $P(\rvs|\rvI,\hat{\rvalpha})$. A reconstruction of the input is then computed as $\hat{\rvI} = \mT(\hat{\rvs})\mPhi \hat{\rvalpha}$, where $\hat{\rvs} = \argmax_\rvs P(\rvs|\rvI,\hat{\rvalpha})$ is the MAP estimate of the $\rvs$. It can be seen from the figure that the inferred $\rvalpha$ is essentially 1-sparse, and that the posterior distribution of $\rvs$ is sharply peaked. 

\begin{figure}[!htb]
\begin{center}
\includegraphics[width=\textwidth]{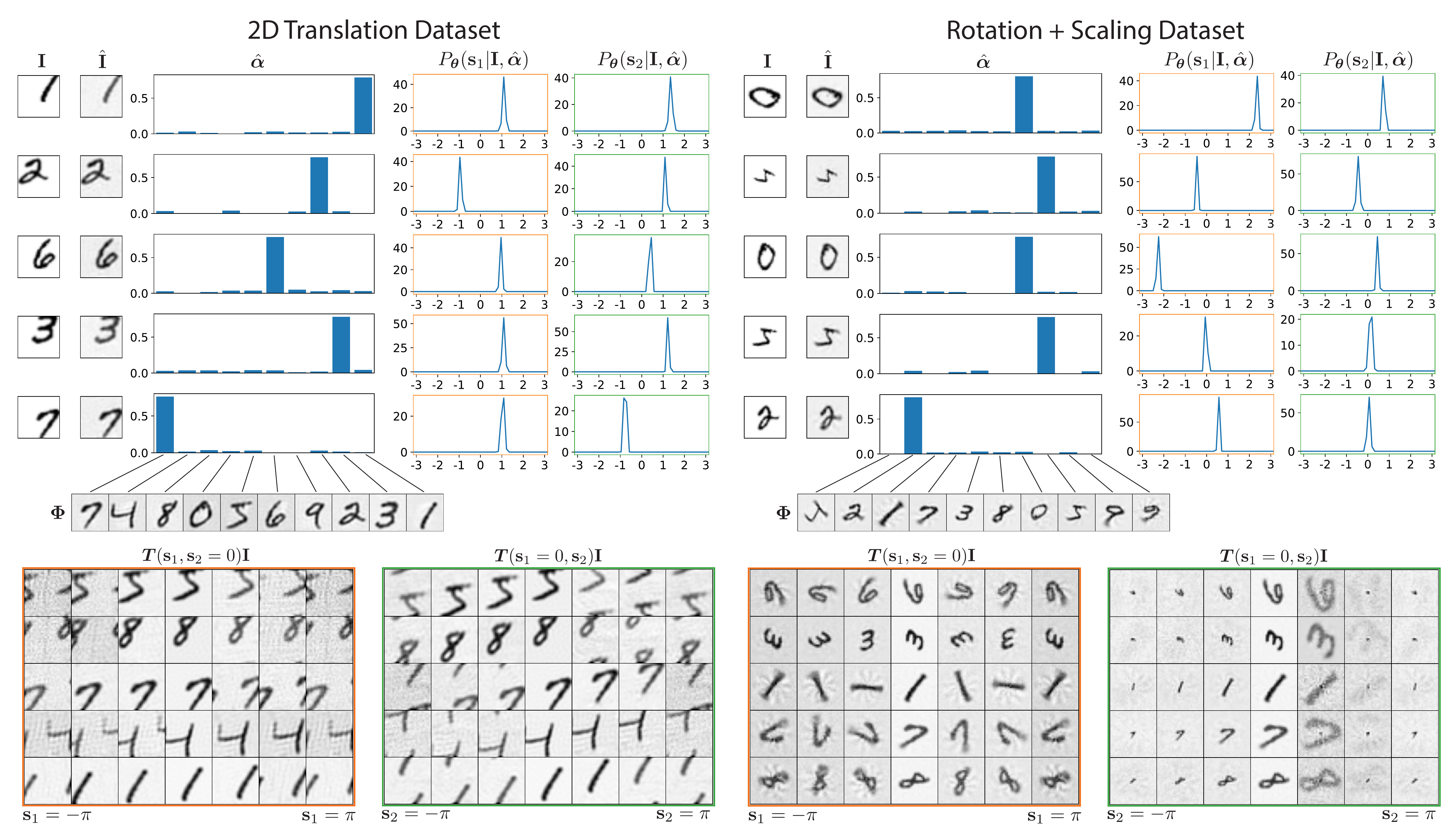}
\end{center}
\caption{Top: Inference and image reconstruction for five inputs $\rvI$ from each dataset. Bottom: Latent traversals of the transformation parameters $\rvs_1$ and $\rvs_2$, obtained by applying $\mT(\rvs)$ with varying values of $\rvs$ to five images from each test set. Orange figure shows latent traversal of $\rvs_1$ from $-\pi$ to $\pi$, while green figure shows latent traversal of $\rvs_2$ from $-\pi$ to $\pi$. The network has been trained on the respective datasets for $20$ epochs.}
\label{fig:3}
\end{figure}

We also trained our model on MNIST to demonstrate its capacity to disentangle shape and transformations on a more natural dataset Where the correct answer is less clear and difficult to ascertain from first principles (see Appendix \ref{section:Appendix-E} for training details). The left side of figure \ref{fig:4} are analogous to figure \ref{fig:3}. We see that LSC learns $9$ out of the $10$ digits with its dictionary $\mPhi$, while the latent traversal plots show that the model has learned horizontal stretching and shearing transformations. Figure \ref{fig:2} also shows the columns of the learned $\mW$ matrix. The reason that the dictionary did not learn the digit ``$1$" is because during inference it always uses the learned horizontal stretching transform to ``squeeze" a ``$0$" into a ``$1$", so that a separate ``$1$" template is not necessary. In our experiments we found that one could learn the ``$1$" template if a prior on $\rvs$ with a narrow peak near $0$ is used instead of a uniform prior, with the intuition being that the narrow prior prevents large horizontal stretching transformations from being used to squeeze a ``$0$" into a ``$1$". 

We compare LSC to sparse coding alone by training it on MNIST as well with the same number of dictionary elements. The learned dictionary $\mPhi$ as well as the inference process is shown on the right of figure $\ref{fig:4}$. For comparison, the same five inputs $\rvI$ were used in inference for both LSC and sparse coding. As can be seen, the reconstruction $\hat{\rvI}$ is much blurrier without a transformation model. The inferred sparse coefficients $\hat{\rvalpha}$ are also less sparse than LSC. Also note that the dictionary $\mPhi$ learned by sparse coding requires more than one template to capture the different poses of a digit, such as the slanted `1' in the 4th dictionary element and the upright `1' in the 7th element. 

\begin{figure}[!htb]
\begin{center}
\includegraphics[width=0.8\textwidth]{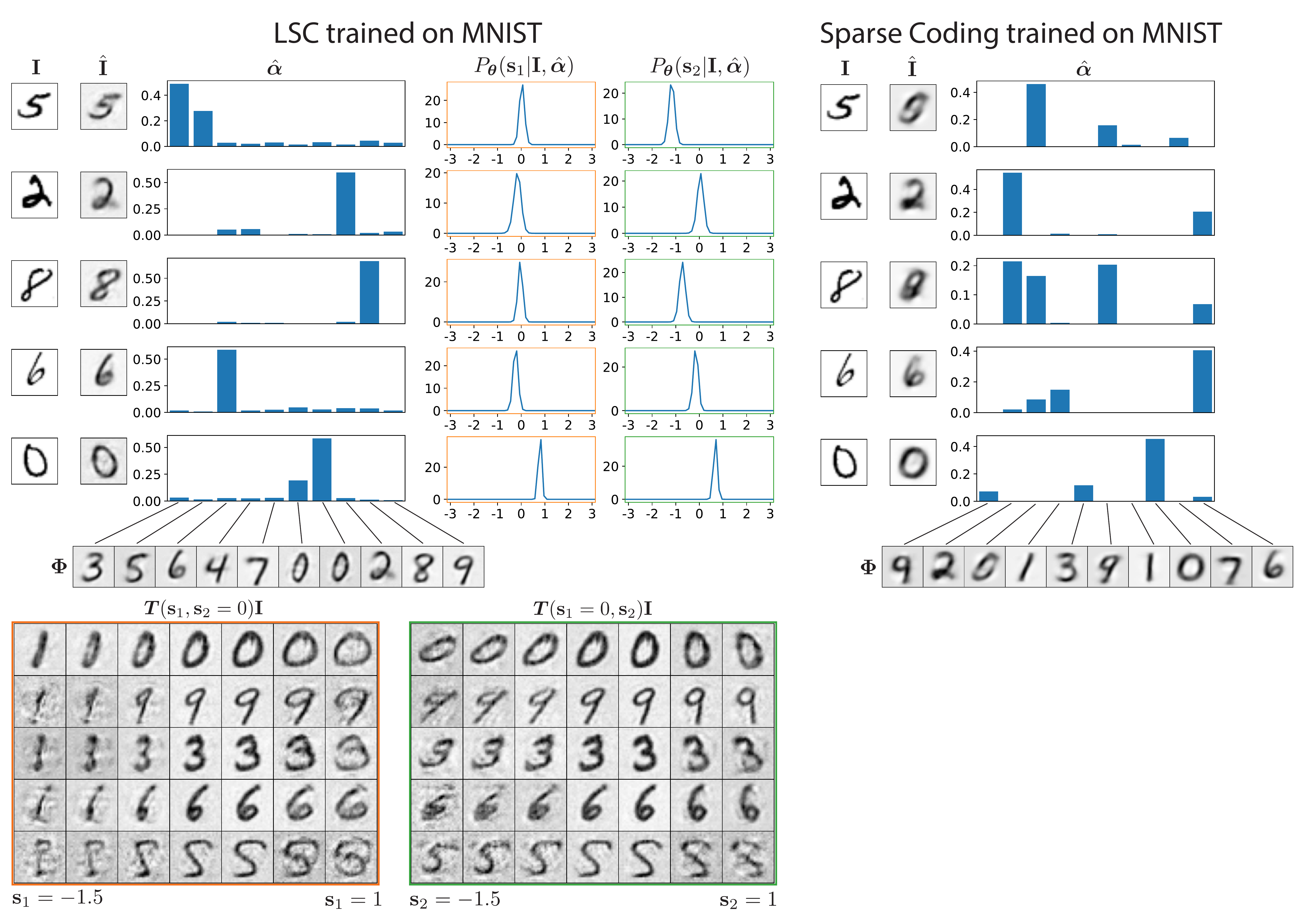}
\end{center}
\caption{Top left: Inference and image reconstruction for five inputs $\rvI$ from MNIST using LSC. Bottom left: Latent traversals of the transformation parameters $\rvs_1$ and $\rvs_2$ using LSC. Orange figure shows latent traversal of $\rvs_1$ from $-1.5$ to $1$, while green figure shows latent traversal of $\rvs_2$ from $-1.5$ to $1$. Top right: Inference and image reconstruction for the same five inputs $\rvI$ using sparse coding. Both LSC and sparse coding has been trained on MNIST for $20$ epochs.}
\label{fig:4}
\end{figure}

We demonstrate the improvement of LSC over sparse coding quantitatively in table \ref{table:1}. We train both algorithms on the same datasets while ensuring that both models are using the same dictionary size and sparsity cost. After training for 20 epochs we evaluate the two algorithms on a test set and calculate the SNR (signal-to-noise ratio) of the reconstructioned images. LSC outperforms sparse coding in all settings, and the improvement in SNR is over 10 times on the 2D translation datset. This is particularly remarkable considering that we set the dimension of the transformation parameter $\rvs$ in LSC to be $n=2$, meaning LSC only has two more degrees of freedom than sparse coding during the inference process. 

\begin{table}[!htb]
\begin{center}
\begin{tabular}{llll}
\multicolumn{1}{c}{\bf Dataset} & \multicolumn{1}{c}{\bf Hyperparamters} &\multicolumn{1}{c}{\bf LSC} &\multicolumn{1}{c}{\bf Sparse Coding}
\\ \hline \\
2D Translation Dataset & Dictionary size 10, sparsity cost $\lambda = 10.0$         & \textbf{28.5} & 2.2 \\
Rotation + Scaling Dataset & Dictionary size 10, sparsity cost $\lambda = 10.0$             & \textbf{20.5} & 2.6\\
MNIST dataset & Dictionary size 10, sparsity cost $\lambda = 10.0$             & \textbf{4.8} & 3.0 \\
MNIST dataset & Dictionary size 100, sparsity cost $\lambda = 1.0$             & \textbf{16.7} & 15.3 \\
\end{tabular}
\end{center}
\caption{Comparison of LSC and sparse coding using SNR (signal-to-noise ratio) of reconstructed images. Each row shows the average SNR of reconstructed images ($\mT(\hat{\rvs})\mPhi\hat{\rvalpha}$ for LSC and $\mPhi \hat{\rvalpha}$ for sparse coding) after training both LSC and sparse coding on the same dataset with the same hyperparameters listed for $20$ epochs (more training details in Appendix \ref{section:Appendix-E})}
\label{table:1}
\end{table}

\section{Discussion}

In this work, we study the problem of disentangling factors of variation in images, specifically discrete patterns vs. continuous transformations, which remains an open theoretical problem. To approach the problem, we combine Lie Group transformation learning and sparse coding within a Bayesian model. We show how spatial patterns and transformations can be learned as separate generative factors and inferred simultaneously. 
We would like to emphasize that the contribution is primarily in theory rather than a specific application. When the model is trained on synthetic MNIST datasets containing known geometric transformations, the digits are learned by the dictionary templates and the applied transformations are learned by the transformation operator, providing a proof of concept demonstration of the feasibility of combining Lie Group transformation learning and sparse coding.



We would like to emphasize two main points about this work. Firstly, building on the foundational work of \citet{Cohen_2014}, we show that incorporating the appropriate mathematical structure for describing transformations (Lie groups) enables a model to learn to disentangle shape and transformations 
in a network with computationally simple structure. The generative model (equation~\ref{eq:generative-model}) is bilinear in the sparse code $\rvalpha$ and block diagonal matrix $\mR$ whose elements in turn are sines and cosines of the transformation variable $\vs$ (equation~\ref{eq:R}).  Although in principle a generic multilayer neural network could learn to approximate the inferential computations in this model, we conjecture that it would lead to a more complicated structure (when evaluated in terms of number of weights and layers) and less robust performance in terms of its ability to generalize outside the training set. The representation theory of Lie groups leads us to a parameterization of the transformation operator that is both computationally efficient and effective at learning a variety of transformations. Secondly, our Bayesian framework provides an advantage for learning a joint form and transformation model. While the usual approach is to jointly optimize the form and transformation parameters using gradient descent, a Bayesian approach reveals that it is better to integrate out the transformations parameters when optimizing the sparse coefficients. Empirically, we found that this approach achieves better convergence than the joint optimization approach, allowing the algorithm to reliably learn the correct transformations and shape dictionary. 

There are two main limitations to our model. First, our current way of computing the gradient of $\rvalpha$ is not scalable to a large number of transformation parameters, as it involves the expectation $\int_\rvs P_\vtheta(\rvs|\rvI,\rvalpha) \mR(\rvs)$, in which the number of samples of $\rvs$ needed to compute the integral scales exponentially with $n$, the dimension of $\rvs$. A possible future direction of our work is to address this issue by finding a simple distribution approximating $P_\vtheta(\rvs|\rvI,\rvalpha)$ that can be used for importance sampling. Another possibility is to use MCMC methods to compute the expectation. 

The second limitation to our model is the various constraints on the transformations that can be learned, which includes orthogonality, compactness, connectedness, and commutativity. Though we showed in our experiments that some of these theoretical constraints are not quite problematic, there are still important transformations which cannot be learned as a result, such as global variations in contrast or a combination of rotation and translation. One possibility is to extend our current method and learn representations of a larger class of Lie groups, but it is unclear whether simple parameterizations exist for such groups. To address this issue, one possible future direction is to learn the representation of an arbitrary Lie group by learning the corresponding Lie algebra. Finally, the key idea of using a trainable Lie Group transformer module instead of a predefined transformation module like the spatial transformer \citep{Jaderberg_2015} may be highly useful for separating the transformation in deep neural networks. We would like to point out this as another interesting future direction.

\subsubsection*{Acknowledgments}
We thank Christian Shewmake for providing many helpful feedbacks during the preparation of this paper. Research conducted by Yubei Chen at Berkley AI Research was funded in part by NSF-IIS-1718991 and NSF-DGE-1106400. Bruno Olshausen’s contributions were funded in part by NSF-IIS-1718991, NSF-DGE-1106400, and DARPA's Virtual Intellgence Processing program (SUPER-HD). Ho Yin Chau’s contributions were funded in part by FAFSA..

\bibliography{references}
\bibliographystyle{style_files/iclr2021_conference}

\section*{Appendix}
\appendix
\section{Orthogonal representations of $\mathbb{T}^n$}
\label{section:Appendix-A}
\begin{theorem}
\label{theorem:1} Any $D$-dimensional real, orthogonal representation of $\mathbb{T}^n$ can be written in the form $\mW\mR(\rvs)\mW^T$ where:
\begin{enumerate}
    \item $\mR(\rvs)$ is a $D \times D$ block-diagonal matrix with $J+1$ blocks for some $J \leq D/2$
    \item The first $J$ blocks of $\mR(\rvs)$ are in $SO(2)$ and correspond to the non-trivial irreducibles.
    \item The last block of $\mR(\rvs)$ is the $(D - 2J) \times (D -2J)$ identity matrix and corresponds to the trivial representation.
    \item $\mW$ is a $D \times D$ orthogonal matrix.
\end{enumerate}
If $D$ is even, then $\mR(\rvs)$ takes the form of Eq. \ref{eq:R}, since in that case $D-2J$ is even, so the last identity matrix block can be written as a direct sum of $SO(2)$ rotation blocks with $\vomega_l = 0$ for each rotation block.
\end{theorem}
\begin{proof}

Any unitary representation $\rho$ of $\mathbb{T}^n$ in $\mathbb{C}^D$ takes the form $\rho(e^{i\rvs}) = \mV e^{\mSigma(\rvs)} \mV^H$ from Eq. \ref{eq:representation}, where $\mSigma(\rvs)$ is a diagonal matrix with diagonal $(-i\vomega_1^T\rvs,\cdots,-i\vomega_D^T\rvs)$.\\
\\
We will show that for every $\vomega_j \neq 0$ there is a $\vomega_k$ such that $\vomega_k = -\vomega_j$.  First, we assume that $\vomega_j \neq 0$ for all $j$. Our first step is to show that there exists some open ball $B \subset \R^n$ such that $e^{-i\vomega_j^T\rvs}$ have non-zero imaginary component for all $j$ and for all $\rvs \in B$. Consider the function:
\[f(\rvs) = \prod_j \mathrm{Log} (e^{-2i\vomega_j^T\rvs}) =  \prod_j[ -2i \vomega_j^T\rvs]
\]
where $\mathrm{Log}$ is the principal branch of the complex logarithm. The second equality follows if we restrict the domain of $f(\rvs)$ to a suitably small open subset $U \subset \mathbb{R}^n$ such that $-\pi < -2\vomega_j^T\rvs < \pi$ for all $j$. From the right hand side, we see $f(\rvs)$ is a non-zero function on $U$ since $\vomega_j$'s are all non-zero by assumption, which implies there must be some open ball $B \subset U$ such that $f(\rvs) \neq 0$ for all $\rvs \in B$ (if not, then the support of $f$, $U \setminus f^{-1}(\{0\})$, has an empty interior, which means the zero set $f^{-1}(\{0\})$ is dense in $U$. But by continuity of $f$, $f^{-1}(\{0\})$ is a closed set, so $f^{-1}(\{0\}) = \overline{f^{-1}(\{0\})} = U$, meaning $f$ is a zero function, contradiction). Since $\mathrm{Log} (e^{-2i\vomega_j^T\rvs}) = 0 \iff (e^{-i\vomega_j^T\rvs})^2 = 1 \iff e^{-i\vomega_j^T\rvs} \in \{1,-1\}$, $f(\rvs) \neq 0$ for all $\rvs \in B$ implies $e^{-i\vomega_j^T\rvs} \not \in \{1, -1\}$ for all $j$ and for all $\rvs \in B$, and so $e^{-i\vomega_j^T\rvs}$ have non-zero imaginary component for all $j$ and for all $\rvs \in B$. 

Now we show that for every $\vomega_j \neq 0$ there is a $\vomega_k$ such that $\vomega_k = -\vomega_j$. Consider the function:
\[g(\rvs) = \prod_{j,k : j < k} \mathrm{Log}\Big(\frac{e^{-i\vomega_j^T\rvs}}{e^{i\vomega_k^T\rvs}}\Big) = \prod_{j,k: j <k}[-i(\vomega_j + \vomega_k)^T \rvs]
\]
where we continue to assume that $\vomega_j \neq 0$ for all $j$. Again, $\mathrm{Log}$ is the principal branch of the complex logarithm, and the second equality follows if we restrict the domain of $g(\rvs)$ to a suitably small open subset of $B' \subset B$ so that $-\pi < -i(\vomega_j + \vomega_k)^T \rvs < \pi$. By Lemma \ref{lemma:1}, every eigenvalue $e^{-i\vomega_j^T\rvs}$ of $\rho(e^{i\rvs})$ has a conjugate eigenvalue. Since $e^{-i\vomega_j^T\rvs}$ has non-zero imaginary part if $\rvs \in B'$, for every $j$ there must be a $k \neq j$ such that  $e^{-i\vomega_k^T\rvs} = e^{i\vomega_j^T\rvs}$ since $e^{-i\vomega_j^T\rvs}$ cannot be conjugate to itself.  Therefore, for all $\rvs \in B'$ at least one of the log factors is 0, so $g(\rvs) = 0$ for all $\rvs \in B'$. Lemma \ref{lemma:2} implies that the polynomial $p_g(\rvs) \equiv \Pi_{j <k}[(\vomega_j + \vomega_k)^T \rvs]$ is the zero polynomial and hence $\vomega_{j^*} = -\vomega_{k^*}$ for some $j^* \neq k^*$. Since conjugate eigenvalues have the same multiplicity by Lemma \ref{lemma:1}, we may remove the term $\log(\frac{e^{-i\vomega_{j^*}^T\rvs}}{e^{i\vomega_{k^*}^T\rvs}})$ from $g(\rvs)$ without changing the fact that $g(\rvs) = 0$. Then, we repeat the above procedure until all $\vomega$'s have been paired. Relaxing the assumption that $\vomega_j \neq 0$ for all $j$, we may apply the above argument by restricting our attention to only the set of non-zero $\vomega$'s. We conclude that for every $\vomega_j \neq 0$ there is a $k$ such that $\vomega_j = -\vomega_k$\\

Therefore, the non-zero $\vomega$'s come in pairs $(\vomega, -\vomega)$. If there are $J$ such pairs, WLOG we assume $\vomega_{2j} = -\vomega_{2j-1}$ for $1 \leq j \leq J$ and $\vomega_k = 0$ for $k > 2J$. Since the columns of $\mV$ in $\rho(e^{i\rvs}) = \mV e^{\mSigma(\rvs)} \mV^H$ are the eigenvectors of $\rho(e^{i\rvs})$ and the eigenvalues $e^{i\vomega_j}$ come in conjugate pairs, Lemma \ref{lemma:1} implies that WLOG we can assume the first $2J$ columns of $\mV$ come in conjugate pairs. Moreover, we may also assume the last $D-2J$ columns also come in conjugate pairs since eigenvectors of real eigenvalues also come in conjugate pairs by Lemma \ref{lemma:1}. Hence, WLOG $V_{2j}  = \overline{V_{2j-1}}$ for all $j$. 

Next, we show that the representation takes the form $\mW\mR(\rvs)\mW^T$. Construct a block diagonal matrix $\mU$ such that the first $J$ blocks are $2 \times 2$ blocks of the form
\[\frac{1}{\sqrt{2}}\begin{bmatrix}
1 & i \\
1 & -i
\end{bmatrix}
\]
and the last block is just a $(D-2J) \times (D-2J)$ identity matrix. Using the fact that $\mU$ is a unitary matrix, we can expand $\rho(e^{i\rvs})$ to get
\[\rho(e^{i\rvs}) = \mV e^{\mSigma(\rvs)}\mV^H = \mV (\mU\mU^H) e^{\mSigma(\rvs)}(\mU\mU^H)\mV^H = (\mV \mU) e^{\mU^H\mSigma(\rvs)\mU}(\mV\mU)^H
\]
For the first $J$ blocks, we restrict our attention to a single $2 \times 2$ block of $\mU^H\mSigma(\rvs)\mU$. Using $\vomega_{2j} = -\vomega_{2j-1}$,  we simplify:
\[\Big(\frac{1}{\sqrt{2}}\begin{bmatrix}
1 & 1 \\
-i & i
\end{bmatrix}\Big)
\begin{bmatrix}
i \vomega_{2j}^T\rvs & 0 \\
0 & -i\vomega_{2j}^T\rvs
\end{bmatrix}
\Big(\frac{1}{\sqrt{2}}\begin{bmatrix}
1 & i \\
1 & -i
\end{bmatrix}\Big) = 
\begin{bmatrix}
0 & \vomega_{2j}^T\rvs \\
-\vomega_{2j}^T\rvs & 0
\end{bmatrix}
\]
The exponential of this block is the $2\times 2$ rotation matrix:
\[\exp({\begin{bmatrix}
0 & \vomega_{2j}^T\rvs \\
-\vomega_{2j}^T\rvs & 0
\end{bmatrix}})= \begin{bmatrix}
\cos(\vomega_{2j}^T\rvs) & -\sin(\vomega_{2j}^T\rvs) \\
\sin(\vomega_{2j}^T\rvs) & \cos(\vomega_{2j}^T\rvs)
\end{bmatrix}
\]
and hence the first $J$ $2 \times 2$ blocks of $e^{\mU^H\mSigma(\rvs)\mU}$ are just $2 \times 2$  rotation matrices. On the other hand, the $(D-2J) \times (D-2J)$ block of $\mU^H \Sigma(\rvs) \mU$ is a zero matrix, so $e^{\mU^H \Sigma(\rvs) \mU}$ is the $(D-2J) \times (D-2J)$ identity matrix. Thus, $e^{\mU^H\mSigma(\rvs)\mU} = \mR(\rvs)$.

Next we show that $\mV \mU$ is a real orthogonal matrix and hence if we let $\mW = \mV \mU$ then $\rho(e^{i\rvs}) = \mW \mR(\rvs) \mW^T$ as desired. Since $\mU^H$ is block-diagonal, when computing the matrix product $\mU^H\mV^H$ we can individually consider each block of $\mU^H$ multiplying its corresponding rows in $\mV^H$. Recall that the columns of $\mV$ come in conjugate pairs. Hence, restricting our attention to one pair of rows in $(\mV\mU)^H = \mU^H \mV^H$, we get
\[\Big(\frac{1}{\sqrt{2}}\begin{bmatrix}
1 & 1 \\
-i & i
\end{bmatrix}\Big)
\begin{bmatrix}
\vv^H \\ \overline{\vv}^H
\end{bmatrix} = \Big(\frac{1}{\sqrt{2}}\begin{bmatrix}
2\mathrm{Re}(\vv^H) \\
2\mathrm{Im}(\vv^H)
\end{bmatrix}\Big)
\]
so all columns of $\mV\mU$ are real. Because $\mV\mU$ is a product of unitary matrices, it is also unitary, and $\mV\mU$ must be real orthogonal Thus we have $\rho(e^{i\rvs}) = \mW \mR(\rvs) \mW^T$.

\end{proof}

\begin{remark}
The converse to the theorem is also true, namely that if $\rho$ is a map from $\mathbb{T}^n$ to $GL(D,\R)$ such that $\rho(e^{i\rvs}) = \mW\mR(\rvs)\mW^T$, then it is a $D$-dimensional real orthogonal representation of $\mathbb{T}^n$. This follows from the fact that $\mW \mR(\rvs)\mW^T$ is orthogonal and $\mR(\rvs_1)\mR(\rvs_2) = \mR(\rvs_1 + \rvs_2)$. Finally, we note that all orthogonal representations of $\mathbb{T}^n$ are actually \textit{special} orthogonal representations of $\mathbb{T}^n$. This is also easy to see, since we know $\det(\rho(1)) = \det(\mathbbm{1}) = 1$, so that if there exists some $e^{i\rvs} \in \mathbb{T}^n$ such that $\det(\rho(e^{i\rvs})) = -1$, then due to continuity of $\det$ and $\rho$ ($\rho$ is smooth by definition of representations of Lie groups), there must exist some $e^{i\rvs'}$ such that $\det(\rho(e^{i\rvs'})) = 0$, which is impossible.
\end{remark}

\begin{lemma}
\label{lemma:1} If $A$ is a real matrix and $\lambda$ is a complex eigenvalue of $A$ with eigenvector $v$, then $\overline{\lambda}$ is also an eigenvalue with eigenvector $\overline{v}$. If $\lambda$ is an eigenvalue with multiplicity $n$, then $\overline{\lambda}$ is also an eigenvalue with multiplicity $n$.
\end{lemma}
\begin{proof}
Since the eigenvalues of $A$ are the roots of its characteristic polynomial $char(A) = \det(A - \lambda I)$, if $A$ is real then $char(A)$ is a real polynomial. Factorizing over $\mathbb{C}$:
\[char(A)(z) = c \ \Pi (z-r_i)
\]
Suppose $r_1$ is a complex root with multiplicity $n > 1$. As $p(\bar{z}) = \overline{p(z)}$ for a real polynomial, $0 = char(A)(r_1) = char(A)(\bar{r_1})$, so $\overline{r_1}$ is also a root and hence an eigenvalue. Removing the factors $(z-r_1)(z-\overline{r_1})$, we repeat the same argument $n$ times to conclude that $\overline{r_1}$ is also an eigenvalue with multiplicity $n$.  \\
\\
If $\lambda$ is a complex eigenvalue of $A$ with eigenvector $v$, then
$A v = \lambda v$. Taking the conjugate of both sides shows:
\[\overline{Av} = A \overline{v} = \overline{\lambda} \overline{v} = \overline{\lambda v} 
\]
\end{proof}
\begin{lemma}
\label{lemma:2} For any $p(\boldsymbol{x}) = p(x_1,\cdots,x_n) \in \mathbb{R}[X_1,\cdots,X_n]$, if $p(\boldsymbol{s}) = 0$ for all $\boldsymbol{s}$ in an open set $S$ then $p(\boldsymbol{x})$ is the zero polynomial.
\end{lemma}
\begin{proof}
Let $ \boldsymbol{t} \in S$. Then, $p'(\boldsymbol{x}) = p(\boldsymbol{x} -  \boldsymbol{t}) = 0$ on some open $S'$ set containing $\boldsymbol{0}$. If two smooth functions coincide over some open set $U$, then their partial derivatives coincide on $U$. Hence, every partial derivative of $p'(\boldsymbol{x}) $ is 0. As evaluating the partial derivatives of $p'(\boldsymbol{x}) $ at $\boldsymbol{0}$ will return its coefficients, every coefficient is 0 and $p'(\boldsymbol{x}) $ is the zero polynomial. As $p(\boldsymbol{x}) = p'(\boldsymbol{x}+t) $, $p(\boldsymbol{x})$ is also the zero polynomial.
\end{proof}

\section{Explicit formulae for $\ln P(I|\alpha)$ and $\ln P(s | I, \alpha)$}
\label{section:Appendix-B}
Although the main results presented in this paper assume a uniform prior on $\rvs$, a more general prior on $\rvs$ can be used, and we will derive the formula using this more general prior. This prior is actually the conjugate prior for our likelihood function, which is desirable as it gives a simple functional form for $\ln P(\rvI|\rvalpha)$. Using this more general prior is likely to be useful in problems where the true prior distribution of the transformation variables is known and can be well-approximated by this prior. Specifically, the general prior on $\rvs$ takes the form:
\[P(\rvs) = \frac{1}{Z(\veta)}\exp\Big(\sum_{l=1}^L \kappa_l \cos(\vomega_l^T\rvs - \mu_l)\Big) = \frac{1}{Z(\veta)}\exp(\veta^T \vT(\rvs))\]
which is a distribution from the exponential family with natural parameter: 
\begin{align*}
\veta &= [\kappa_1\cos(\mu_1), \kappa_1\sin(\mu_1), \cdots, \kappa_L\cos(\mu_L), \kappa_L\sin(\mu_L)]^T \\
&\equiv [\veta_{11},\veta_{12}, \cdots, \veta_{L1}, \veta_{L2}]^T
\end{align*}
sufficient statistics: 
\[\vT(\rvs) = [\cos(\vomega_1^T\rvs), \sin(\vomega_1^T\rvs), \cdots, \cos(\vomega_L^T\rvs), \sin(\vomega_L^T\rvs)]^T,\]
and normalization constant: 
\[Z(\veta) = \int_\rvs \exp(\veta^T \vT(\rvs)) = \int_0^{2\pi}\cdots\int_0^{2\pi} \exp(\veta^T \vT(\rvs)) d\rvs_1 \cdots d\rvs_n.\]
Note that one can recover the uniform prior on $\rvs$ by simply taking $\kappa_l = 0$ for all $l$. We also note that this is only a slightly more generalized version of the prior discovered by \citet{Cohen_2014}, and that the following derivation of $\ln P(\rvI | \rvalpha)$ is also entirely due to \citet{Cohen_2014}, with only slight modifications to adapt to our new model.

According to our model $\rvI = \mW \mR(\rvs) \mW^T \mPhi \rvalpha + \vepsilon$, where $\vepsilon \sim \mathcal{N}(0,\sigma^2\mathbbm{1})$, we have:
\[P(\rvI|\rvs,\rvalpha) = \frac{1}{(2\pi\sigma^2)^{D/2}}\exp\Big(-\frac{||\rvI - \mW \mR(\rvs) \mW^T \mPhi \rvalpha||_2^2}{2\sigma^2}\Big)\]

Now, for convenience, define $\bm{u} = \mW \mPhi \rvalpha$ and $\vv = \mW \rvI$. Moreover, define
\[\hat{\veta} = [\hat{\eta}_{11}, \hat{\eta}_{12}, \hat{\eta}_{21}, \hat{\eta}_{22},\cdots, \hat{\eta}_{L1},\hat{\eta}_{L2}]
\]
such that
\[
\begin{bmatrix}\hat{\eta}_{l1} \\ \hat{\eta}_{l2}\end{bmatrix} 
= \begin{bmatrix}\eta_{l1} \\ \eta_{l2}\end{bmatrix}
+ \frac{1}{\sigma^2}
\begin{bmatrix}
u_{l1}v_{l1}+u_{l2}v_{l2} \\ 
u_{l1}v_{l2}-u_{l2}v_{l1}
\end{bmatrix}
\]
where
\begin{align*}
\veta &= [\eta_{11}, \eta_{12}, \eta_{21}, \eta_{22},\cdots, \eta_{L1},\eta_{L2}] \\
\bm{u} &= [u_{11}, u_{12}, u_{21}, u_{22}, \cdots, u_{L1}, u_{L2}]
\\
\vv &= [v_{11}, v_{12}, v_{21}, v_{22}, \cdots, v_{L1}, v_{L2}]
\end{align*}

Then:
\begin{align*}
\ln P(\rvI | \rvalpha)
&= \ln \int_\rvs P(\rvI | \rvs, \rvalpha) P(\rvs) \\
&= \ln \int_\rvs \frac{1}{(2\pi\sigma^2)^{D/2}}\exp\Big(-\frac{||\rvI - \mW \mR(\rvs) \mW^T \mPhi \rvalpha||_2^2}{2\sigma^2}\Big)\frac{1}{Z(\veta)}\exp(\veta^T\vT(\rvs)) \\
&= \ln \int_\rvs \frac{1}{(2\pi \sigma^2)^{D/2}}\exp(-\frac{1}{2\sigma^2}(||\mW^T\mPhi\rvalpha||_2^2 + ||\rvI||_2^2) + \frac{1}{\sigma^2}\vv^T\mR(\rvs)\bm{u})\frac{1}{Z(\veta)}  \exp(\veta^T\vT(\rvs)) \\
&= \ln \Big(\frac{\exp(-\frac{1}{2\sigma^2}(||\mW^T\mPhi\rvalpha||_2^2 + ||\rvI||_2^2))}{(2\pi \sigma^2)^{D/2}}\frac{1}{Z(\veta)} \int_\rvs \exp(\veta^T\vT(\rvs)+\frac{1}{\sigma^2}\vv^T\mR(\rvs)\bm{u}) \Big) \\
&= -\frac{1}{2\sigma^2}(||\mW^T\mPhi\rvalpha||_2^2 + ||\rvI||_2^2) - \frac{D}{2}\ln(2\pi\sigma^2) - \ln Z(\veta) + \ln \Big(\int_\rvs \exp(\hat{\veta}^T\vT(\rvs)) \Big) \\
&= -\frac{1}{2\sigma^2}(||\mW^T\mPhi\rvalpha||_2^2 + ||\rvI||_2^2) - \frac{D}{2}\ln(2\pi\sigma^2) + \ln Z(\hat{\veta}) - \ln Z(\veta)
\end{align*}
where the integrands in step 4 and 5 are equal because

\begin{align*}
\veta^T\vT(\rvs)+\frac{1}{\sigma^2}\vv^T\mR(\rvs)\bm{u} 
&= \sum_{l=1}^L 
\begin{bmatrix}
\eta_{l1} \\ \eta_{l2}
\end{bmatrix}^T
\begin{bmatrix}
\cos(\vomega_l^T\rvs) \\ \sin(\vomega_l^T\rvs)
\end{bmatrix}
+ 
\frac{1}{\sigma^2}
\begin{bmatrix}
v_{l1} \\ v_{l2}
\end{bmatrix}^T
\begin{bmatrix}
\cos(\vomega_l^T\rvs)&-\sin(\vomega_l^T\rvs)\\
\sin(\vomega_l^T\rvs)&\cos(\vomega_l^T\rvs)
\end{bmatrix}
\begin{bmatrix}
u_{l1} \\ u_{l2}
\end{bmatrix} \\
&= \sum_{l=1}^L 
\Big(\begin{bmatrix}\eta_{l1} \\ \eta_{l2}\end{bmatrix}
+ \frac{1}{\sigma^2}
\begin{bmatrix}
u_{l1}v_{l1}+u_{l2}v_{l2} \\ 
u_{l1}v_{l2}-u_{l2}v_{l1}
\end{bmatrix}\Big)^T 
\begin{bmatrix}
\cos(\vomega_l^T\rvs) \\ \sin(\vomega_l^T\rvs)
\end{bmatrix} \\
&= \hat{\veta}^T\mT(\rvs)
\end{align*}

Note that the parameter $\hat{\veta}$ determines the posterior distribution of $\rvs$, which is given by
\[P(\rvs | \rvI, \rvalpha) = \frac{1}{Z(\hat{\veta})}\exp(\hat{\veta}^T\vT(\rvs)),\]
since $P(\rvs | \rvI, \rvalpha) \propto P(\rvI | \rvs, \rvalpha)P(\rvs) \propto \exp(\hat{\veta}^T\vT(\rvs))$ (the last step can be seen from the derivation of $\ln P(\rvI|\rvalpha)$ where $P(\rvI | \rvs, \rvalpha) P(\rvs)$ is the integrand). In order to compute $P(\rvs | \rvI, \rvalpha)$, we just compute $\hat{\veta}$ and then apply the formula.

\section{Gradients of the model}
\label{section:Appendix-C}
In this section we provide details of the derivations for Eq. \ref{eq:gradient_theta}, Eq. \ref{eq:gradient_theta_exact}, and \ref{eq:gradient_alpha}, as well as concrete expressions for the gradients. Notice that the derivations for Eq. \ref{eq:gradient_theta}, Eq. \ref{eq:gradient_theta_exact}, and \ref{eq:gradient_alpha} all employ the same "trick."

Eq. \ref{eq:gradient_theta}:
\begin{align*}
\nabla_\vtheta \ln P_\vtheta(\rvI)
&= \frac{1}{P_\vtheta(\rvI)} \nabla_\vtheta P_\vtheta(\rvI) \\
&= \frac{1}{P_\vtheta(\rvI)} \nabla_\vtheta \int_{\rvs, \rvalpha} P_\vtheta(\rvI, \rvs, \rvalpha) \\
&= \int_{\rvs, \rvalpha}\frac{P_\vtheta(\rvI, \rvs, \rvalpha)}{P_\vtheta(\rvI)} \nabla_\vtheta \ln P_\vtheta(\rvI, \rvs, \rvalpha) \\
&= \int_{\rvs, \rvalpha}P_\vtheta(\rvs, \rvalpha|\rvI) \nabla_\vtheta (\ln P_\vtheta(\rvI| \rvs, \rvalpha) + \ln P(\rvs,\rvalpha))\\
&= \int_\rvalpha \Big(P_\vtheta(\rvalpha|\rvI) \int_\rvs P_\vtheta(\rvs|\rvI, \rvalpha) \nabla_\vtheta \ln P_\vtheta(\rvI| \rvs, \rvalpha)\Big)\\
&\approx \int_\rvalpha \Big(\delta(\rvalpha - \hat{\rvalpha}) \int_\rvs P_\vtheta(\rvs|\rvI, \rvalpha) \nabla_\vtheta \ln P_\vtheta(\rvI| \rvs, \rvalpha)\Big)\\
&= \int_\rvs P_\vtheta(\rvs|\rvI, \hat{\rvalpha}) \nabla_\vtheta \ln P_\vtheta(\rvI| \rvs, \hat{\rvalpha})\\
&= \E_{\rvs \sim P_\vtheta(\rvs|\rvI,\hat{\rvalpha})}[\nabla_\vtheta \ln P_\vtheta(\rvI|\rvs,\hat{\rvalpha})]
\end{align*}

Eq. \ref{eq:gradient_theta_exact}:
\begin{align*}
\nabla_\vtheta \ln P_\vtheta(\rvI | \hat{\rvalpha})
&= \frac{1}{P_\vtheta(\rvI | \hat{\rvalpha})} \nabla_\vtheta P_\vtheta(\rvI | \hat{\rvalpha}) \\
&= \frac{1}{P_\vtheta(\rvI | \hat{\rvalpha})} \nabla_\vtheta \int_\rvs P_\vtheta(\rvI, \rvs| \hat{\rvalpha}) \\
&= \int_\rvs \frac{P_\vtheta(\rvI, \rvs| \hat{\rvalpha})}{P_\vtheta(\rvI | \hat{\rvalpha})} \nabla_\vtheta \ln P_\vtheta(\rvI, \rvs| \hat{\rvalpha}) \\
&= \int_\rvs P_\vtheta(\rvs| \rvI, \hat{\rvalpha}) \nabla_\vtheta (\ln P_\vtheta(\rvI| \rvs, \hat{\rvalpha}) + \ln P(\rvs| \hat{\rvalpha})) \\
&= \int_\rvs P_\vtheta(\rvs| \rvI, \hat{\rvalpha}) \nabla_\vtheta \ln P_\vtheta(\rvI| \rvs, \hat{\rvalpha})\\
&= \E_{\rvs \sim P_\vtheta(\rvs|\rvI,\hat{\rvalpha})}[\nabla_\vtheta \ln P_\vtheta(\rvI|\rvs,\hat{\rvalpha})]
\end{align*}

Eq. \ref{eq:gradient_alpha}:
\begin{align*}
\nabla_\rvalpha \ln P_\vtheta(\rvalpha | \rvI)
&= \nabla_\rvalpha \ln P_\vtheta(\rvI | \rvalpha) + \nabla_\rvalpha \ln P(\rvalpha) \\
&= \E_{\rvs \sim P_\vtheta(\rvs|\rvI,\rvalpha)}[\nabla_\rvalpha \ln P_\vtheta(\rvI|\rvs,\rvalpha)] + \nabla_\rvalpha \ln P(\rvalpha)
\end{align*}
where in the last step we applied the Eq. \ref{eq:gradient_theta_exact} with $\hat{\rvalpha}$ replaced by $\rvalpha$.

Next we provide concrete expressions for the various gradients:
\begin{align*}
\nabla_{\rvalpha} \ln P_\vtheta(\rvalpha|\rvI) 
&= \E_{\rvs \sim P_\vtheta(\rvs|\rvI,\rvalpha)}[\nabla_{\rvalpha} \ln P_\vtheta(\rvI|\rvs,\rvalpha)] + \nabla_{\rvalpha} \ln P(\rvalpha) \\
&= \E_{\rvs \sim P_\vtheta(\rvs|\rvI,\rvalpha)}[\frac{1}{\sigma^2} \mPhi^T\mT(\rvs)^T\vepsilon] - \lambda \mathrm{sign}(\rvalpha) \\
&= \frac{1}{\sigma^2} \mPhi^T\mW(\bar{\mR}^T\mW^T\rvI - \mW^T\mPhi\rvalpha) - \lambda \mathrm{sign}(\rvalpha)
\end{align*}
where $\vepsilon = \rvI - \mT(\rvs)\mPhi\rvalpha$ and $\bar{\mR} = \E_{\rvs \sim P_\vtheta(\rvs|\rvI,\rvalpha)}[\mR(\rvs)]$.

\begin{align*}
\nabla_\mPhi \ln P_\vtheta(\rvI) 
&\approx \E_{\rvs \sim P_\vtheta(\rvs|\rvI,\hat{\rvalpha})} [\nabla_\mPhi \ln P_\vtheta(\rvI|\rvs, \hat{\rvalpha})] \\
&= \E_{\rvs \sim P_\vtheta(\rvs|\rvI,\hat{\rvalpha})} [\frac{1}{\sigma^2}\mT(\rvs)^T\hat{\vepsilon}\hat{\rvalpha}^T] \\
&= \frac{1}{\sigma^2}\mW(\bar{\mR}^T\mW^T\rvI - \mW^T\mPhi\hat{\rvalpha})\hat{\rvalpha}^T
\end{align*}
where $\hat{\vepsilon} = \rvI - \mT(\rvs)\mPhi\hat{\rvalpha}$.

\begin{align*}
\nabla_\mW \ln P_\vtheta(\rvI) 
&\approx \E_{\rvs \sim P_\vtheta(\rvs|\rvI,\hat{\rvalpha})} [\nabla_\mW \ln P_\vtheta(\rvI|\rvs, \hat{\rvalpha})] \\
&= \E_{\rvs \sim P_\vtheta(\rvs|\rvI,\hat{\rvalpha})} [\frac{1}{\sigma^2}(\hat{\vepsilon}(\mT(\rvs)\mPhi\hat{\rvalpha})^T+ \mPhi\hat{\rvalpha}(\mT(\rvs)^T\hat{\vepsilon})^T)\mW] \\
&= \frac{1}{\sigma^2}(\mPhi\hat{\rvalpha}(\mW^T\rvI)^T\bar{\mR} + \rvI(\mW^T\mPhi\hat{\rvalpha})^T\bar{\mR}^T - \mPhi\hat{\rvalpha}(\mW^T\mPhi\hat{\rvalpha})^T - \E_{\rvs \sim P_\vtheta(\rvs|\rvI,\hat{\rvalpha})}[\hat{\rvI}\hat{\rvI}^T]\mW) \\
&= \frac{1}{\sigma^2}(\mPhi\hat{\rvalpha}(\bar{\mR}^T\mW^T\rvI)^T + \rvI(\bar{\mR}\mW^T\mPhi\hat{\rvalpha})^T - \mPhi\hat{\rvalpha}(\mW^T\mPhi\hat{\rvalpha})^T \\
& \hspace{0.5in} - \mW\E_{\rvs \sim P_\vtheta(\rvs|\rvI,\hat{\rvalpha})}[\mR(\rvs)\mW^T\mPhi\hat{\rvalpha}\hat{\rvalpha}^T\mPhi^T\mW\mR(\rvs)^T])
\end{align*}
where $\hat{\rvI} = \mT(\rvs)\mPhi\hat{\rvalpha}$

As seen, the gradient for $\mW$ is quite difficult to compute. In our implementation, we used an approximation that results in a much simpler expression for the gradient for $\mW$ derived above, by assuming independence between the term $\mT(\rvs)$ and $\hat{\vepsilon}$, which are both dependent on the random variable $\rvs$. For consistency, we also applied the same approximation to both $\rvalpha$ and $\mPhi$ gradients. More explicitly, we use the following approximate gradients:
\begin{align*}
\nabla_{\rvalpha} \ln P_\vtheta(\rvalpha|\rvI)
&= \E_{\rvs \sim P_\vtheta(\rvs|\rvI,\rvalpha)}[\frac{1}{\sigma^2} \mPhi^T\mT(\rvs)^T\vepsilon] - \lambda \mathrm{sign}(\rvalpha) \\
&\approx \frac{1}{\sigma^2} \mPhi^T\bar{\mT}^T\hat{\bar{\vepsilon}} - \lambda \mathrm{sign}(\rvalpha) \\
\end{align*}
\begin{align*}
\nabla_\mPhi \ln P_\vtheta(\rvI) 
&\approx
\E_{\rvs \sim P_\vtheta(\rvs|\rvI,\hat{\rvalpha})} [\frac{1}{\sigma^2}\mT(\rvs)^T\hat{\vepsilon}\hat{\rvalpha}^T] \\
&\approx
\frac{1}{\sigma^2}\bar{\mT}^T\hat{\bar{\vepsilon}}\hat{\rvalpha}^T \\
\end{align*}
\begin{align*}
\nabla_\mW \ln P_\vtheta(\rvI) 
&\approx
\E_{\rvs \sim P_\vtheta(\rvs|\rvI,\hat{\rvalpha})} [\frac{1}{\sigma^2}(\hat{\vepsilon}(\mT(\rvs)\mPhi\hat{\rvalpha})^T+ \mPhi\hat{\rvalpha}(\mT(\rvs)^T\hat{\vepsilon})^T)\mW] \\
&\approx
\frac{1}{\sigma^2}(\hat{\bar{\vepsilon}}(\bar{\mT}\mPhi\hat{\rvalpha})^T+ \mPhi\hat{\rvalpha}(\bar{\mT}^T\hat{\bar{\vepsilon}})^T)\mW
\end{align*}
where $\bar{\mT} = \E_{\rvs \sim P_\vtheta(\rvs|\rvI,\hat{\rvalpha})}[\mT(\rvs)]$ and $\hat{\bar{\vepsilon}} = \E_{\rvs \sim P_\vtheta(\rvs|\rvI,\hat{\rvalpha})}[\hat{\epsilon}]$.

We found by chance that the approximate gradient works better than the exact gradient in practice. However, we currently do not have a theory for why the approximate gradient works better.


\section{Usage of FISTA in LSC}
\label{section:Appendix-D}
FISTA is a method for fast gradient descent when the objective function is a sum of a smooth convex function $f$ and a non-smooth convex function $g$ \citep{Beck_2009}. It is typically applied to problems such as the traditional sparse coding, where the objective is the sum of the smooth convex function $||\rvI - \mPhi \rvalpha||_2^2$ and non-smooth convex sparsity cost $||\rvalpha||_1$. As an extension of sparse coding, we would like to use FISTA in order to speed up convergence for $\rvalpha$ as well. The main problem is that our new objective is possibly a non-convex function of $\rvalpha$, and as a result the theoretical guarantees of FISTA may not apply. Fortunately, we find that despite the lack of theoretical guarantees, FISTA still works very well in LSC. 

In LSC, we directly applied FISTA with constant step size to perform the optimization problem $$\argmin_\rvalpha - \ln P_\vtheta(\rvalpha| \rvI) = \argmax_\rvalpha (- \ln P_\vtheta(\rvI| \rvalpha) - \ln P(\rvalpha)) \equiv \argmax_\rvalpha (f(\rvalpha) + g(\rvalpha))$$
where $g$ corresponds to the non-smooth convex function assumed in the FISTA paper. The only free parameter is the choice of step size, which, accoridng to FISTA, should be set as $1/L(f)$ if $f(\rvalpha)$ were convex, where $L(f)$ is a Lipschitz constant of $f$. In our case, we set the step size as $1.5||\frac{1}{\sigma^2}\mPhi^TWW^T\mPhi||$, where $||\cdot||$ is the spectral norm of the matrix. To understand this choice of step size, we begin by noting that 
\begin{align*}
\nabla_{\rvalpha} f(\rvalpha) 
&= -\nabla_{\rvalpha} \ln P_\vtheta(\rvI|\rvalpha) \\
&= \E_{\rvs \sim P_\vtheta(\rvs|\rvI,\rvalpha)}[-\nabla_{\rvalpha} \ln P_\vtheta(\rvI|\rvs,\rvalpha)] \\
&= \int_\rvs P_\vtheta(\rvs | \rvI, \rvalpha) \vh(\rvs, \rvalpha)
\end{align*}
where $\vh(\rvs, \rvalpha) \equiv -\nabla_{\rvalpha} \ln P_\vtheta(\rvI|\rvs,\rvalpha) = - \frac{1}{\sigma^2} \mPhi^T\mT(\rvs)^T(\rvI - \mT(\rvs)\mPhi\rvalpha)$ is the score.
Hence
\begin{align*}
\nabla_{\rvalpha}^2 f(\rvalpha)
&= \nabla_\rvalpha \int_\rvs P_\vtheta(\rvs | \rvI, \rvalpha) \vh(\rvs, \rvalpha) \\
&= \int_\rvs \vh(\rvs, \rvalpha) \nabla_\rvalpha P_\vtheta(\rvs | \rvI, \rvalpha)^T + P_\vtheta(\rvs | \rvI, \rvalpha) \nabla_\rvalpha \vh(\rvs, \rvalpha) \\
&= \int_\rvs P_\vtheta(\rvs|\rvI, \rvalpha) [\vh(\rvs, \rvalpha) \nabla_\rvalpha \ln P_\vtheta(\rvs | \rvI, \rvalpha)^T + \nabla_\rvalpha \vh(\rvs, \rvalpha)] \\
&= \int_\rvs P_\vtheta(\rvs|\rvI, \rvalpha) \vh(\rvs, \rvalpha) (\nabla_\rvalpha \ln P_\vtheta(\rvI| \rvs, \rvalpha) - \nabla_\rvalpha \ln P_\vtheta(\rvI | \rvalpha))^T \\
& \hspace{0.5in} + \int_s P_\vtheta(\rvs|\rvI,\rvalpha) \frac{1}{\sigma^2}\mPhi^TWW^T\mPhi \\
&= \frac{1}{\sigma^2}\mPhi^TWW^T\mPhi - \E_{s \sim P_\vtheta(\rvs|\rvI, \rvalpha)} [\vh(\rvs, \rvalpha) \vh(\rvs, \rvalpha)^T] + \nabla_\rvalpha \ln P_\vtheta(\rvI | \rvalpha) \nabla_\rvalpha \ln P_\vtheta(\rvI | \rvalpha)^T \\
&= \frac{1}{\sigma^2}\mPhi^TWW^T\mPhi - \E_{s \sim P_\vtheta(\rvs|\rvI, \rvalpha)} [\vh(\rvs, \rvalpha) \vh(\rvs, \rvalpha)^T] \\
& \hspace{0.5in} + \E_{s \sim P_\vtheta(\rvs|\rvI, \rvalpha)} [\vh(\rvs, \rvalpha)]\E_{s \sim P_\vtheta(\rvs|\rvI, \rvalpha)} [\vh(\rvs, \rvalpha)]^T \\
&= \frac{1}{\sigma^2}\mPhi^TWW^T\mPhi - \Cov(\vh(\rvs, \rvalpha))
\end{align*}
Since $\mPhi^TWW^T\mPhi$ is positive semidefinite and $- \Cov(\vh(\rvs, \rvalpha))$ is negative semidefinite, the sum is not guaranteed to be positive semidefinite. Hence $f(\rvalpha)$ is not necessarily convex. However, it does imply that $$||\nabla_\rvalpha^2 f(\rvalpha)|| \leq ||\frac{1}{\sigma^2}\mPhi^TWW^T\mPhi|| + ||- \Cov(\vh(\rvs, \rvalpha))|| \leq ||\frac{1}{\sigma^2}\mPhi^TWW^T\mPhi|| + || \Cov(\vh(\rvs, \rvalpha))||$$
If $f(\rvalpha)$ were convex, then this would mean that a Lipschitz constant for $f$ would be $||\frac{1}{\sigma^2}\mPhi^TWW^T\mPhi|| + M$, where $M$ is a bound for $|| \Cov(\vh(\rvs, \rvalpha))||$. $M$ is difficult to compute, but in practice we find that setting the step size as $1.5||\frac{1}{\sigma^2}\mPhi^TWW^T\mPhi||$ works well, which suggests $M$ is usually small in comparison to the first term. This is not surprising, especially since during the later stages of training the variance of $\rvs$ is usually very small, and hence the covariance of $\vf(\rvs, \rvalpha)$ is also expected to have a small spectral norm.

\section{Training details}
\label{section:Appendix-E}
For training on the 2D translation and the rotation + scaling dataset, the hyperparameters used for the model is detailed in Table \ref{params}.

\begin{table}[!htb]
\label{params}
\begin{center}
\begin{tabular}{ll}
\multicolumn{1}{c}{\bf Variable}  &\multicolumn{1}{c}{\bf Value}
\\ \hline \\
B (batch size) & 100 \\
K (number of dictionary templates)         & 10 \\
N (number of samples along each dimension of the integral $\bar{\mR} = \int_\rvs P_\vtheta(\rvs|\rvI,\rvalpha) \mR(\rvs)$)             & 50 \\
L (number of irreducible representations)             & 128 \\
$T$ (number of gradient update steps for $\rvalpha$) & 20 \\
$n$ (dimensionality of transformation parameter $\rvs$) & 2\\
$\sigma^2$ (variance of the Gaussian noise $\epsilon$ in the generative model) & 0.01 \\
$\lambda$ (sparse penalty) & 10 \\
$\eta_\mPhi$ (learning rate for $\mPhi$) & 0.05 \\
$\eta_\mW$ (learning rate for $\mW$) & 0.3 \\
$\rvalpha_0$ (initialization of $\rvalpha$) & $0.01$ \\
multiplicity of $\vomega$ & $1$ \\
parameters for geoopt Riemannian ADAM optimizer (excluding learning rate) & default \\
\end{tabular}
\end{center}
\caption{Hyperparameters of LSC when trained on 2D translation and rotation + scaling datasets}
\end{table}

When trained on MNIST, the only change is that the multiplicity of $\vomega$ is $2$ instead of $1$.

For computing the SNRs in table \ref{table:1}, we used $N = 100$ instead of $N = 50$ to obtain higher quality reconstruction. The SNRs for sparse coding are obtained using an improved version of the algorithm in \citet{Olshausen_1997}, where the main modification is the use of an approximate second-order gradient descent method for optimizing $\mPhi$. Specifically, instead of updating $\mPhi$ with the gradient step $\Delta \mPhi = \eta_\mPhi \nabla_\mPhi L$, we update it with the approximate second-order gradient descent step $\Delta \mPhi_k = \frac{1}{\overline{\alpha_k^2}+\epsilon}(\eta_\mPhi \nabla_\mPhi L)_k$, where $\mPhi_k$ is the $k$-th column of $\mPhi$, $\overline{\alpha_k^2}$ is the average of $\alpha_k^2$ over the last $300$ batches of $\rvalpha$, and $\epsilon = 0.001$ is a small constant added to prevent instability. This greatly speeds up convergence of the dictionary. To ensure fair comparison between LSC and sparse coding, the following hyperparameters for both models are set to be equal: B, K, $\sigma^2$, and  $\lambda$. Notice that when these hyperparameters are set equal, the loss function for LSC coincides with the loss function for sparse coding in the limiting case where all the $\vomega_l$ are $0$ (in which case the transformation $\mT(\rvs)$ will just be the identity matrix) and $\mW$ is full rank. The rest of the hyperparameters for sparse coding algorithm are manually optimized for best possible SNRs.

\end{document}